\documentclass[10pt,twocolumn,letterpaper]{article}

\usepackage[T1]{fontenc}
\usepackage[margin=0.75in,columnsep=0.25in]{geometry}
\usepackage[hyphens]{url}
\usepackage{graphicx}
\usepackage{microtype}
\usepackage{mathptmx}
\usepackage{natbib}
\usepackage{caption}
\usepackage{amsmath}
\usepackage{amssymb}
\usepackage{amsthm}
\usepackage{booktabs}
\usepackage{multirow}
\usepackage{array}

\usepackage[colorlinks=true,linkcolor=blue,citecolor=blue,urlcolor=blue]{hyperref}
\hypersetup{pdftitle={Attention Sensitivity Is Not Enough: Dissociating Attention-Level and Behavioural In-Context Learning under Fine-Tuning},pdfauthor={Jinyuan Zhang, Peng He, Yin Yuan, He Hu, ShengShuo Jiao}}
\usepackage[capitalize,noabbrev]{cleveref}

\newcommand{\E}{\mathbb{E}}

\newcommand{\Prob}{\mathbb{P}}

\newcommand{\Loss}{\mathcal{L}}
\newcommand{\eqdef}{\triangleq}
\newcommand{\indic}{\mathbf{1}}
\newcommand{\inner}[2]{\langle #1,\,#2\rangle}

\DeclareMathOperator*{\argmax}{arg\,max}
\DeclareMathOperator{\acc}{acc}

\DeclareMathOperator{\softmax}{softmax}

\DeclareMathOperator{\KL}{KL}

\newcommand{\Dmatch}{\mathcal{D}_{1}}
\newcommand{\Dmismatch}{\mathcal{D}_{2}}
\newcommand{\ICS}{\mathrm{ICS}}
\newcommand{\ICShat}{\widehat{\ICS}}
\newcommand{\ICLgap}{\mathrm{ICL\text{-}GAP}}
\newcommand{\Lfunc}{\Loss_{\mathrm{func\text{-}icl}}}
\newcommand{\Lce}{\Loss_{\mathrm{CE}}}
\newcommand{\Ltot}{\Loss_{\mathrm{total}}}

\newcommand{\armA}{\textsc{F\_A}}
\newcommand{\armM}{\textsc{F\_M}}
\newcommand{\armNone}{\textsc{F\_None}}
\newcommand{\armKL}{\textsc{F\_ICS-Max}}
\newcommand{\armKLone}{\textsc{F\_ICS-Max}\textsubscript{v1}}
\newcommand{\armKLtwo}{\textsc{F\_ICS-Max}\textsubscript{v2}}

\newcommand{\BICS}{\mathrm{B\text{-}ICS}}
\newcommand{\BICSbeta}{\mathrm{B\text{-}ICS}_{\!\beta}}
\newcommand{\Lanchor}{\Loss_{\mathrm{anchor}}}
\newcommand{\Lbics}{\Loss_{\mathrm{B\text{-}ICS}}}
\newcommand{\Lkll}{\Loss_{\mathrm{KL\text{-}logits}}}
\newcommand{\Llwt}{\Loss_{\ell_{2}\text{-}\theta_{0}}}
\newcommand{\Sanchor}{\mathcal{S}_{\mathrm{anchor}}}
\newcommand{\armAnchor}{\textsc{F\_anchor}}
\newcommand{\armBICS}{\textsc{F\_BICS}}
\newcommand{\armKLL}{\textsc{LogitAnchor}}
\newcommand{\armLtwo}{\textsc{WeightAnchor}}

\newcommand{\armAnchorLast}{\textsc{F\_anchor\_last4}}
\newcommand{\armAnchorLM}{\textsc{F\_anchor\_lmhead}}

\graphicspath{{figures/}}
\theoremstyle{plain}
\newtheorem{theorem}{Theorem}
\newtheorem{proposition}[theorem]{Proposition}

\theoremstyle{definition}

\begin{document}

\twocolumn[{%
  \begin{center}
    {\Large\bfseries Attention Sensitivity Is Not Enough:\\
     Dissociating Attention-Level and Behavioural\\
     In-Context Learning under Fine-Tuning\par}
    \vspace{1.3em}
    {\normalsize
    \begin{tabular}{c@{\hspace{4em}}c}
      Jinyuan Zhang & Peng He$^{\ast}$\\
      {\ttfamily 202621116012480@stu.hubu.edu.cn} & {\ttfamily penghe@hubu.edu.cn}\\
      \addlinespace[3pt]
      Yin Yuan & He Hu\\
      {\ttfamily 202521120012766@stu.hubu.edu.cn} & {\ttfamily 202521120012751@stu.hubu.edu.cn}\\
      \addlinespace[3pt]
      ShengShuo Jiao & \\
      {\ttfamily 202621120012764@stu.hubu.edu.cn} & \\
    \end{tabular}\par}
    \vspace{1.1em}
    {\itshape Hubei University, Wuhan, China\par}
    \vspace{0.5em}
    {$^{\ast}$Corresponding author: {\ttfamily penghe@hubu.edu.cn}\par}
    \vspace{1.2em}
    \begin{minipage}{0.92\textwidth}
      \begin{abstract}
In-context learning (ICL) lets large language models adapt to new tasks from demonstrations, and fine-tuning can erode this behaviour. Many preservation diagnostics inspect attention: if attention changes when demonstrations change, the model is treated as context-sensitive. This paper asks how far that proxy can be trusted once it is optimised. We formalise \emph{In-Context Sensitivity} (ICS), the average row distance between last-token attention on matched and mismatched demonstration prefixes, and pair it with \emph{ICL-GAP}, the behavioural accuracy gap between the same prefixes. In a controlled four-arm ablation on Llama-2-7B, an ICS-maximising regulariser ($\armKL$) drives ICS to $1.413$, within $0.5\%$ of its geometric ceiling. The behavioural readout tells a different story: ICL-GAP stays near zero and MMLU accuracy moves from $0.371$ to $0.279$, a Goodhart dissociation of the bounded attention proxy. Endpoint statistics locate the mechanism: attention grows sharp and near-disjoint across prefixes yet routes to formatting and demonstration-body tokens rather than labels. A random-label protocol confirms that the behavioural probe family retains dynamic range at the same checkpoints. In a constructive sweep, behaviour gating partially mitigates the effect, while objectives anchored to pretrained computation hold the high-MMLU, moderate-ICS region that divergence maximisers leave. The main lesson is diagnostic: attention-level ICL proxies earn their place as training targets only after validation against behavioural gaps.

      \end{abstract}
    \end{minipage}
  \end{center}
  \vspace{2.2em}
}]

\section{Introduction}
\label{sec:intro}

Large language models can solve new tasks from demonstrations in their input, a capability known as in-context learning (ICL) \citep{brown2020language,wei2022emergent}. ICL requires no parameter updates, but subsequent fine-tuning can alter the behaviour on which this adaptation relies. Continued pretraining and instruction tuning have been observed to weaken ICL on held-out tasks \citep{luo2023empirical,shi2024continual,wang2023incontextlearning}, motivating methods that measure and limit this drift.

Many diagnostics inspect attention because ICL has been associated with mechanisms such as induction heads and copying circuits \citep{olsson2022context,akyurek2022learning}. A common approach compares attention maps under a matched prefix containing correct demonstrations and a mismatched prefix containing shuffled or random demonstrations. The resulting divergence measures whether attention responds to the demonstrations. It is inexpensive and differentiable, which makes it suitable for monitoring and, potentially, regularisation---though on its own it says nothing about whether that response improves task behaviour.

\begin{figure}[t]
\centering
\includegraphics[width=\linewidth]{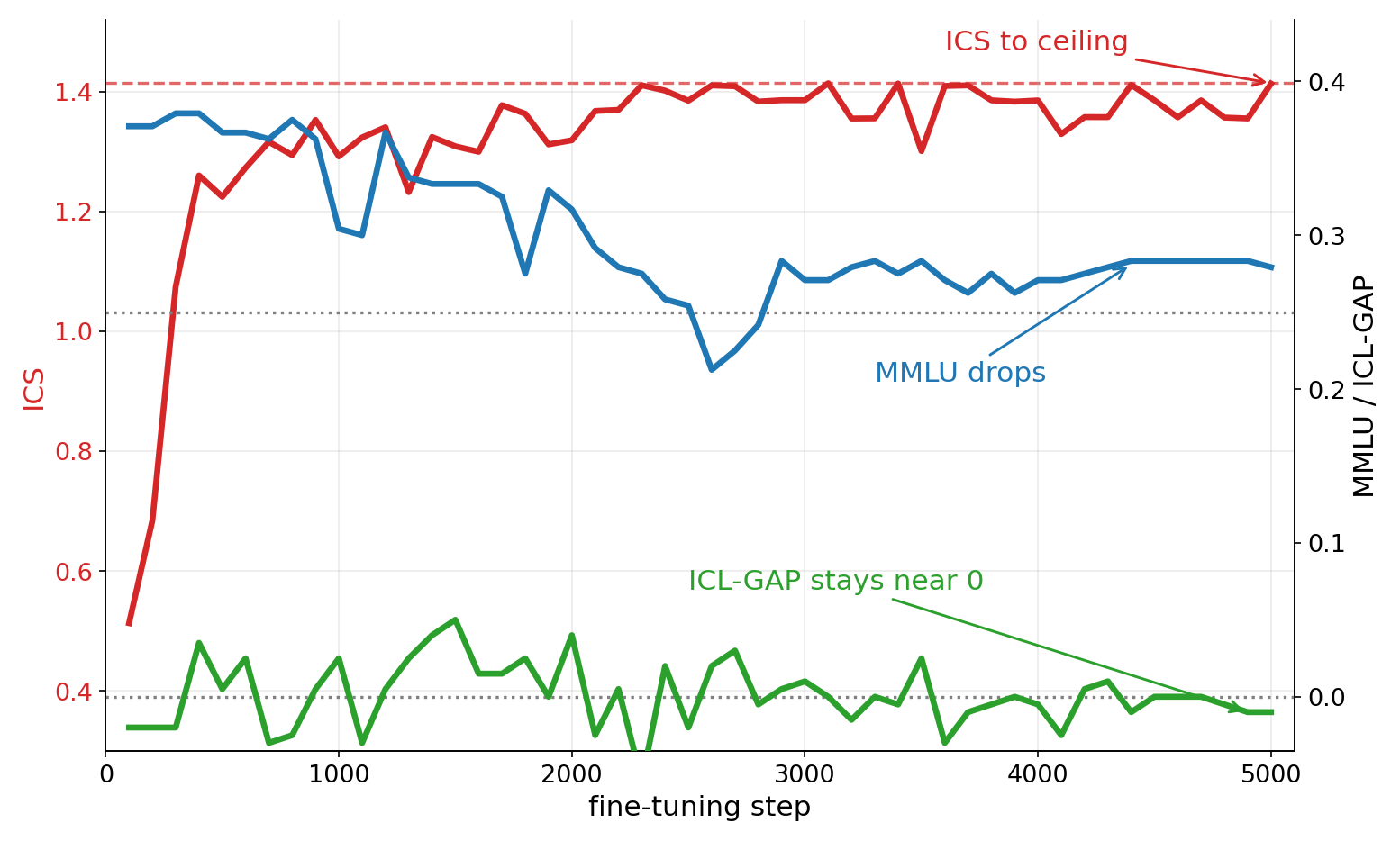}
\caption{Under the ICS-maximising stress test, the attention proxy approaches its ceiling while ICL-GAP stays near zero and MMLU declines. This divergence motivates testing the proxy against behavioural ICL under optimisation.}
\label{fig:motivation}
\end{figure}

We test whether an attention-level ICL proxy remains faithful when it becomes an optimisation target. Attention maps are already used in distillation, alignment, and faithfulness analysis \citep{jiao2020tinybert,tropeano2026dontbreakingllmimpact,yao2026adaptattentiondynamicsalignment}, making differentiable attention probes plausible auxiliary preservation signals. Without attributing the exact ICS objective to prior work, we use a matched-vs.-mismatched attention diagnostic to test the broader assumption that increasing context-sensitive attention preserves behavioural ICL.

\paragraph{Setup and central finding.}
We formalise the proxy as \emph{In-Context Sensitivity} (ICS): the mean distance, over last-token attention rows in a band of mid-network layers, between a matched and a mismatched demonstration prefix. Because each attention row is a probability distribution, ICS has a fixed geometric ceiling, attained only when the two prefixes induce attention concentrated on disjoint single tokens. Its behavioural counterpart is \emph{ICL-GAP}: the held-out task accuracy under the matched prefix minus the held-out accuracy under the mismatched prefix. ICL-GAP is the behavioural quantity any ICL-preservation method actually wants to keep above zero.

On Llama-2-7B we run a four-arm controlled ablation, with all arms starting from the same checkpoint and training for 5,000 steps on a fixed instruction-tuning mixture. The first three arms freeze attention, freeze MLPs, or fine-tune all parameters without an auxiliary loss; the fourth, \armKL{}, adds an ICS-maximising term to the cross-entropy loss. The unregularised arms remain close to the pretrained ICS of 0.516. \armKL{} reaches 1.413, within 0.5 percent of the theoretical ceiling of about 1.414, while ICL-GAP on a held-out QA probe remains near zero and MMLU accuracy falls from 0.371 to 0.279.

\paragraph{Main empirical finding.}
Considered alone, the $2.74\times$ increase in attention divergence could be read as stronger context sensitivity. The behavioural measurement points the other way, and the gap between the two readings is exactly why attention-level ICL diagnostics need validation against a behavioural gap before they are used under optimisation pressure.

\paragraph{Contributions.}
We make four contributions. First, we validate a matched-vs.-mismatched behavioural probe and define ICS as a geometrically bounded attention diagnostic paired with behavioural ICL-GAP. Second, a controlled Llama-2-7B ablation shows that explicit ICS maximisation saturates the proxy while leaving ICL-GAP near zero and degrading MMLU. Third, attention endpoint statistics and a structural construction show how sharp, disjoint routing can maximise ICS without selecting answer-relevant tokens. Finally, we test behaviour-gated B-ICS and pretrained anchoring as constructive responses to this dissociation. The supplement gives the contamination postmortem, full trajectories, layer ablations, and constructive sweeps.

\section{Related Work}
\label{sec:related}

\paragraph{What ICL is, and where it lives.}
In-context learning emerged as an empirical property of large transformers \citep{brown2020language,wei2022emergent}, and mechanistic work has since tied it to specific circuitry: induction heads that copy an earlier pattern's continuation forward \citep{olsson2022context}, and attention layers that implement gradient-descent-like or Bayesian updates over the prefix \citep{akyurek2022learning,vonoswald2023transformers,xie2022explanation}. Recent probes cut the other way: in tabular and structured-data settings, apparent ICL can reduce to memorisation or label priors \citep{capano2026probingmemorizationtabularincontext,pelusi2026categoricalpriorlockinincontext}.
% TODO-cite: task-vector / function-vector accounts of ICL (Hendel et al. 2023; Todd et al. 2024), which localise ICL to compact internal representations beyond induction heads.
These accounts disagree about what ICL is, but share one implication: intact ICL should respond to demonstrations in attention. We do not contest that direction; we put the converse under optimisation pressure---that context-sensitive attention implies intact ICL---once the response becomes a training signal.

\paragraph{ICL under fine-tuning.}
Continued training moves this behaviour, usually tracked at the output. Instruction tuning, supervised fine-tuning, and continual pretraining can all reduce held-out ICL accuracy \citep{luo2023empirical,shi2024continual,wang2023incontextlearning}, and finer-grained work follows how fine-tuning reshapes in-context factual recall \citep{huang2026finetuningdynamicsincontextfactual}.
% TODO-cite: instruction-tuning work reporting improved few-shot ICL (Wei et al. 2022 FLAN; Sanh et al. 2022 T0; Chung et al. 2022), so the erosion claim is framed as one direction of a mixed literature.
The measurement is behavioural throughout: ICL accuracy, or accuracy gaps between matched and mismatched prefixes. Our ICL-GAP belongs to that family; we depart by instrumenting an attention-level diagnostic, ICS, alongside it and studying the relationship between the two under training.

\paragraph{Attention as evidence.}
Attention maps are tempting evidence: cheap, differentiable, already load-bearing. Distillation transfers them between models \citep{jiao2020tinybert}; pruning attention layers changes explanation faithfulness and confidence \citep{tropeano2026dontbreakingllmimpact}; alignment work optimises attention dynamics as a preference signal \citep{yao2026adaptattentiondynamicsalignment}. Yet attention weights need not track the computation behind a prediction.
% TODO-cite: attention-as-explanation debate (Jain and Wallace 2019; Wiegreffe and Pinter 2019; Serrano and Smith 2019)---the canonical critique that attention weights are not, on their own, explanations.
Our stress test sharpens that caution: not whether attention explains a fixed model, but whether an attention proxy stays meaningful while being optimised.

\paragraph{Component-level attribution.}
A parallel line localises capabilities to subnetworks. Mechanistic interpretability isolates circuits at the level of heads and MLP neurons \citep{elhage2021mathematical,wang2023interpretability}; parameter-efficient tuning trains only attention or only MLP blocks \citep{houlsby2019parameter,he2022unifiedview}. Our $\armA$ and $\armM$ arms repurpose the second design for a stability question: which subnetwork's update carries ICL drift? In our setup neither does alone---both preserve ICS near its pretrained value---motivating an explicit proxy regulariser instead of a structural constraint.

\paragraph{Auxiliary objectives under Goodhart pressure.}
Regularisers that protect old behaviour assume the auxiliary signal tracks the capability. EWC \citep{kirkpatrick2017overcoming}, synaptic intelligence \citep{zenke2017continual}, and KL-to-base penalties \citep{ouyang2022training} anchor to different quantities, but each substitutes a measurable surrogate for the behaviour itself. Goodhart's law \citep{goodhart1984problems,manheim2018categorizing} names the failure mode; the RLHF literature gives modern instances, from reward-model over-optimisation \citep{skalse2022defining,gao2023scaling} to reward hacking at inference time and beyond \citep{khalaf2025inferencetimerewardhackinglarge,fu2026rewardshapingmitigatereward,liu2026largelanguagemodelshack}.
% TODO-cite: specification-gaming / Goodhart-in-ML surveys (e.g., Krakovna et al. 2020), which catalogue this failure across learning systems.
Those warnings concern reward functions; we give the same failure an attention-level instance, driving ICS within $0.5\%$ of its geometric ceiling while behavioural ICL stays near zero.

\paragraph{Calibration as a side signal.}
Confidence can part ways with behaviour as well. Fine-tuning shifts calibration \citep{guo2017calibration,desai2020calibration}; we record ECE \citep{naeini2015obtaining} on MMLU \citep{hendrycks2020measuring} as a diagnostic only, since lower ECE under near-chance predictions says little about reasoning quality.

\paragraph{Where this work sits.}
ICL-preservation methods fall into three families: freeze a parameter subset, replay ICL-format data, or regularise toward base-model outputs.
% TODO-cite: replay-based continual learning (e.g., Rolnick et al. 2019) for the second family.
Our $\armA$, $\armM$, and $\armNone$ arms cover the first family in controlled form and our logit-anchoring baseline the third; replay is discussed but not completed in the main comparison. AnchorTune is closest to attention distillation \citep{jiao2020tinybert,jiang2024minillm}, but anchors the trainee to its own pretrained checkpoint $\theta_0$ on last-token attention rows over a fixed ICL probe. We present it as one member of an anchored family whose anchor matches the diagnostic under test, not as a uniquely superior method; B-ICS adds the complement, gating the proxy by behaviour rather than tying it to a checkpoint.

\section{Methodology}
\label{sec:method}

This section defines the diagnostic framework used in the rest of the paper. The goal is deliberately narrow: separate an attention-level sign of context sensitivity from the behavioural effect that an ICL-preservation method is meant to preserve. We first define the matched-vs.-mismatched problem setting and its diagnostic metrics, then introduce the proxy-maximisation stress test and two constructive variants. A separate random-label protocol validates the behavioural measurement before the main stress-test results.

\subsection{Problem Setting}
\label{subsec:problem-setting}

Let $\mathcal{M}_\theta$ be a pretrained autoregressive transformer. For a labelled task $(x,y)\sim p(x,y)$ with label set $\mathcal{Y}$, draw $k=4$ demonstrations $S=\{(x_i,y_i)\}_{i=1}^{k}$ and a random label permutation $\pi$. For the same query $x$, we compare a matched prefix $\Dmatch(x;S)$ with a mismatched prefix $\Dmismatch(x;S,\pi)$ that preserves the token template but breaks the input--label correspondence. A model with behaviourally intact ICL should satisfy
\begin{equation}
\acc_{\Dmatch}(\theta)>\acc_{\Dmismatch}(\theta).
\label{eq:icl-intact}
\end{equation}
This matched-vs.-mismatched contrast is the common protocol behind both the attention diagnostic and the behavioural probe.

\subsection{Diagnostic Metrics}
\label{subsec:diagnostic-metrics}

For a prefix $\mathrm{D}\in\{\Dmatch,\Dmismatch\}$, layer $l$, and head $h$, write the last-token attention row as
\begin{equation}
a_{l,h}(\mathrm{D};\theta)
\eqdef
\softmax\!\left(
\frac{q_{l,h}(\mathrm{D})K_{l,h}(\mathrm{D})^{\top}}{\sqrt{d_h}}+m
\right).
\label{eq:attn-row}
\end{equation}
We define \emph{In-Context Sensitivity} (ICS) by averaging the distance between the matched and mismatched rows over a probe distribution $\mu$ and a layer band $\mathcal{L}$:
\begin{equation}
\ICS(\theta)
\eqdef
\E_{\mu}\!\left[
\frac{1}{|\mathcal{L}|H}
\sum_{l\in\mathcal{L}}\sum_{h=1}^{H}
d_{l,h}(x,S,\pi;\theta)
\right],
\label{eq:ics}
\end{equation}
where
\begin{equation}
d_{l,h}(x,S,\pi;\theta)
\eqdef
\bigl\|a_{l,h}(\Dmatch;\theta)
-a_{l,h}(\Dmismatch;\theta)\bigr\|_{2}.
\label{eq:ics-row-distance}
\end{equation}
For Llama-2-7B we use $\mathcal{L}=\{8,\dots,16\}$; the supplement ablates this choice.

\begin{proposition}[Geometric ceiling]
\label{prop:ceiling}
For any $\theta$, $\ICS(\theta)\in[0,\sqrt{2}]$. For any attention rows $p,q$,
\begin{equation}
\|p-q\|_2^2
=
\|p\|_2^2+\|q\|_2^2-2\inner{p}{q}
\le 2,
\label{eq:row-bound}
\end{equation}
with equality iff $p$ and $q$ are one-hot on disjoint coordinates. Thus $\ICS=\sqrt{2}$ requires every probed matched/mismatched attention-row pair to be one-hot and disjoint almost surely.
\end{proposition}

ICS measures whether attention reorganises when demonstrations change. The behavioural target is the accuracy gap induced by the same intervention. Let $\hat y_\theta(\mathrm{D})=\argmax_{c\in\mathcal{Y}}\Prob_\theta(c\mid\mathrm{D})$. We define
\begin{align}
\acc_{\mathrm{D}}(\theta)
&\eqdef
\E_{\mu}\!\left[\indic\{\hat y_\theta(\mathrm{D})=y\}\right],
\label{eq:prefix-acc}\\
\ICLgap(\theta)
&\eqdef
\acc_{\Dmatch}(\theta)-\acc_{\Dmismatch}(\theta).
\label{eq:gap}
\end{align}
The central question is whether high ICS implies positive ICL-GAP under optimisation pressure. We also report MMLU accuracy and Expected Calibration Error (ECE):
\begin{equation}
\mathrm{ECE}
\eqdef
\sum_{b=1}^{B}\frac{|B_b|}{N}\,
\bigl|\acc(B_b)-\mathrm{conf}(B_b)\bigr|.
\label{eq:ece}
\end{equation}

\subsection{Proxy-Maximisation Stress Test}
\label{subsec:stress-test}

\begin{figure*}[t]
\centering
\includegraphics[width=0.95\textwidth]{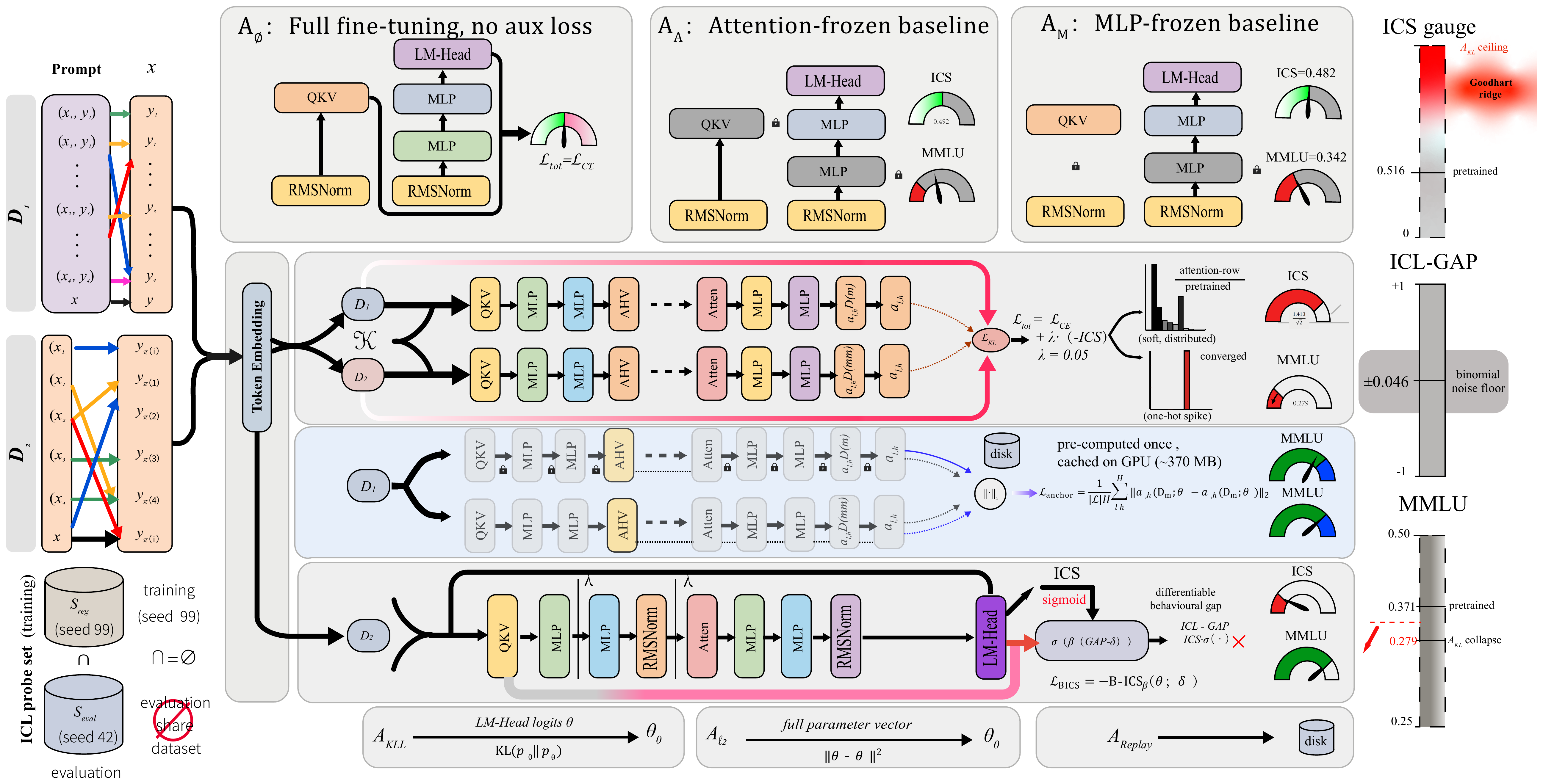}
\caption{Diagnostic framework. Matched and mismatched prompts feed the same fine-tuning pipeline, while the attention-level ICS and behavioural ICL-GAP probes are computed on a held-out pair set disjoint from any regulariser pairs. The architecture also shows the two constructive directions studied later: behaviour-gated diagnostics and anchoring to the pretrained computation.}
\label{fig:architecture}
\end{figure*}

To test the proxy rather than merely observe it, we add an ICS-maximising auxiliary objective to the ordinary cross-entropy loss. Let
\begin{equation}
\mathcal{S}_{\mathrm{reg}}
=
\{(x^{(n)},S^{(n)},\pi^{(n)})\}_{n=1}^{N_{\mathrm{reg}}}
\label{eq:reg-set}
\end{equation}
be a held-out regulariser set, and let $\ICShat_{\mathcal{S}_{\mathrm{reg}}}$ be the empirical estimate of \cref{eq:ics} on that set. The stress-test arm minimises
\begin{align}
\Ltot(\theta)
&\eqdef \Lce(\theta)+\lambda\,\Lfunc(\theta),
\label{eq:total-loss}\\
\Lfunc(\theta)
&\eqdef -\,\ICShat_{\mathcal{S}_{\mathrm{reg}}}(\theta),
\label{eq:fkl-loss}
\end{align}
with $\lambda=0.05$. If ICS is a faithful proxy for ICL, this arm should improve the matched-vs.-mismatched behavioural gap; any proxy gain without a matching behavioural gain then measures the proxy's Goodhart exposure directly.

\subsection{Constructive Variants}
\label{subsec:constructive-variants}

We evaluate two compact variants aimed at mitigating this dissociation. B-ICS gates the attention proxy by a differentiable behavioural sentinel:
\begin{equation}
\BICSbeta(\theta;\delta)
\eqdef
\ICS(\theta)\,
\sigma\!\bigl(\beta(\widetilde{\ICLgap}(\theta)-\delta)\bigr),
\label{eq:bics-smooth}
\end{equation}
where
\begin{equation}
\widetilde{\ICLgap}(\theta)
\eqdef
\E_\mu\!\left[
\log\Prob_\theta(y\mid\Dmatch)
-\log\Prob_\theta(y\mid\Dmismatch)
\right].
\label{eq:soft-gap}
\end{equation}
The soft gap serves as a differentiable surrogate for the discrete ICL-GAP, without any claim that log-probability gaps and accuracy gaps agree pointwise.

AnchorTune instead anchors the trainee's matched-prefix attention rows to the pretrained model $\theta_0$ on $\Sanchor$:
\begin{equation}
\Lanchor
\eqdef
\E_{\Sanchor,l,h}\!\left[
\bigl\|a_{l,h}(\Dmatch;\theta)
-a_{l,h}(\Dmatch;\theta_0)\bigr\|_2
\right],
\label{eq:l-anchor}
\end{equation}
with objective
\begin{equation}
\Ltot(\theta)
\eqdef
\Lce(\theta)+\lambda\,\Lanchor(\theta;\theta_0,\Sanchor).
\label{eq:anchortune-objective}
\end{equation}
Unlike a divergence-maximising loss, this objective is minimised at the pretrained attention pattern rather than on the disjoint-support ridge of Proposition~\ref{prop:ceiling}.

\section{Experiments and Analysis}
\label{sec:experiments}

We organise the empirical study around four research questions. They progress from validating the behavioural measurement, through establishing and explaining the proxy--behaviour dissociation, to testing constructive responses:
\begin{itemize}
    \item \textbf{RQ0.} Does the pretrained model exhibit behavioural ICL under a controlled matched-vs.-mismatched protocol?
    \item \textbf{RQ1.} Can an attention-level ICL proxy be driven to its geometric ceiling while behavioural ICL remains absent?
    \item \textbf{RQ2.} What changes inside the model when the proxy is maximised, and why does this fail to improve behaviour?
    \item \textbf{RQ3.} Can behavioural gating or pretrained anchoring reduce this proxy Goodhart failure?
\end{itemize}
After a shared experimental setup, the following modules answer these questions in order. Each module presents question-specific evidence and closes with a direct answer.

\subsection{Shared Training and Evaluation Setup}
\label{subsec:experimental-protocol}

All main arms start from the same Llama-2-7B pretrained checkpoint \citep{touvron2023llama} and run for $5{,}000$ optimiser steps on the same instruction-tuning mixture. The arms share the optimiser, schedule, sequence length, batch size, precision, and evaluation cadence; exact training details are in the supplement. Probes are evaluated every $100$ steps, and the supplement reports full trajectories.

The controlled ablation compares four arms. $\armA$ freezes attention projections and updates the MLP, layer norms, and language-model head. $\armM$ freezes the MLP blocks and updates attention projections, layer norms, and the head. $\armNone$ full-finetunes all parameters with cross-entropy only. $\armKL$ full-finetunes all parameters with the ICS-maximising stress-test loss in \cref{eq:fkl-loss}.

To avoid probe contamination, the regulariser and evaluation pair sets are disjoint:
\begin{equation}
\mathcal{S}_{\mathrm{reg}}\cap\mathcal{S}_{\mathrm{eval}}=\emptyset.
\label{eq:disjoint-probes}
\end{equation}
The regulariser set has $500$ pairs sampled with seed $99$; the evaluation set has $500$ pairs sampled with seed $42$. ICS and ICL-GAP are always reported on $\mathcal{S}_{\mathrm{eval}}$. MMLU uses a fixed $240$-item subset stratified over $57$ subjects, evaluated under a $5$-shot prefix at temperature $0$. ECE uses the MMLU predictions with $15$ equal-mass bins. GSM8K exact-match, which sits at zero for the base model across arms, is excluded from the main evidence.

\subsection{Does the Behavioural Probe Detect ICL? (RQ0)}
\label{subsec:s0}

A proxy-stress test is only meaningful if the base model can use demonstrations in at least one held-out protocol. We therefore screened pretrained Llama-2-7B on $480$ random-label binary classification episodes. Each episode samples a semantic split, assigns the two classes arbitrary labels, and compares three prompts:
\begin{equation}
\mathrm{D}\in\{\Dmatch,\,\Dmismatch,\,\mathrm{D}_{\emptyset}\}.
\label{eq:s0-prompts}
\end{equation}
The model scores only the two candidate label tokens. Pretrained accuracy is $0.669$ with matched demonstrations, $0.363$ with swapped demonstrations, and $0.492$ without demonstrations. The paired bootstrap estimate is
\begin{align}
\acc_{\Dmatch}-\acc_{\Dmismatch}
&=0.306,
\label{eq:s0-gap}\\
95\%~\mathrm{CI}
&=[0.231,0.379].
\label{eq:s0-ci}
\end{align}
Thus the checkpoint has a clear behavioural ICL effect on this sanity-check protocol. We keep this result separate from the controlled QA probe, which is harder and has a near-zero behavioural gap across arms.

To check that this behavioural measurement family still has dynamic range after fine-tuning, we reran the same RQ0 protocol on the step-$5{,}000$ checkpoints of all four arms. \Cref{tab:rq0-all-arms} shows that every arm retains a positive matched-vs.-permuted gap with a bootstrap confidence interval excluding zero. This does not rescue the controlled QA gap; rather, it rules out the simpler objection that all behavioural probes are insensitive at these checkpoints. In particular, $\armKL$ reaches the largest RQ0 gap while leaving the controlled QA ICL-GAP near zero.

\begin{table}[t]
\centering
\footnotesize
\setlength{\tabcolsep}{2.5pt}
\caption{RQ0 random-label A/B protocol at step $5{,}000$. All arms retain a clear matched-vs.-permuted behavioural gap on this protocol, so the near-zero controlled QA ICL-GAP cannot be attributed to a universally insensitive behavioural probe family.}
\label{tab:rq0-all-arms}
\begin{tabular}{lcccc}
\toprule
\textbf{Arm} & \textbf{Match} & \textbf{Perm.} & \textbf{None} & \textbf{Gap [95\% CI]} \\
\midrule
\armA & 0.663 & 0.360 & 0.492 & 0.302 [0.235, 0.371] \\
\armM & 0.656 & 0.381 & 0.492 & 0.275 [0.206, 0.346] \\
\armNone & 0.623 & 0.435 & 0.492 & 0.188 [0.115, 0.260] \\
\armKL & 0.692 & 0.340 & 0.492 & 0.352 [0.292, 0.413] \\
\bottomrule
\end{tabular}
\setlength{\tabcolsep}{6pt}
\end{table}

\paragraph{Answer to RQ0.}
Yes. The pretrained checkpoint shows a clear matched-vs.-mismatched behavioural effect, and the same probe family retains dynamic range after fine-tuning. This validation does not imply that the harder controlled QA probe must be positive; it establishes that its near-zero gap is not caused by universal probe insensitivity.

\subsection{Can ICS Saturate without Behavioural ICL? (RQ1)}
\label{subsec:rq1}

Before the controlled ablation, we ran a coarse full-fine-tuning baseline under an earlier data mixture and a higher learning rate. Starting from $\ICS=0.516$, that run collapsed to $\ICS\approx0.20$ at step $5{,}000$, a relative reduction of $61\%$. It serves only as motivation---an attention-proxy analogue of ICL fragility \citep{luo2023empirical} rather than a behavioural-collapse claim.

\Cref{tab:main} reports the step-$5{,}000$ values of ICS, ICL-GAP, MMLU accuracy, and ECE for the controlled ablation, with the pretrained baseline in the first row. ICL-GAP was logged during training for $\armKL$ and later evaluated post-hoc for the unregularised arms.

\begin{table*}[t]
\centering
\caption{Final controlled-ablation results at step $5{,}000$. ICS is the attention-level diagnostic from \cref{eq:ics}, bounded above by $\sqrt{2}\approx 1.414$. ICL-GAP is the behavioural quantity from \cref{eq:gap}. MMLU accuracy is on a $240$-item probe (chance $0.25$). The pretrained baseline ICS is $0.516$.}
\label{tab:main}
\begin{tabular}{lcccc}
\toprule
\textbf{Arm} & \textbf{ICS} ($\uparrow$) & \textbf{ICL-GAP} ($\uparrow$) & \textbf{MMLU acc.} ($\uparrow$) & \textbf{ECE} ($\downarrow$) \\
\midrule
Pretrained Llama-2-7B & 0.516 & --- & --- & --- \\
\armA & 0.492 & --- & 0.338 & 0.434 \\
\armM & 0.482 & --- & 0.342 & 0.406 \\
\armNone & 0.500 & --- & 0.375 & 0.407 \\
\armKL & \textbf{1.413} & $-0.010$ & 0.279 & 0.231 \\
\bottomrule
\end{tabular}
\end{table*}

\begin{figure}[t]
\centering
\includegraphics[width=\linewidth]{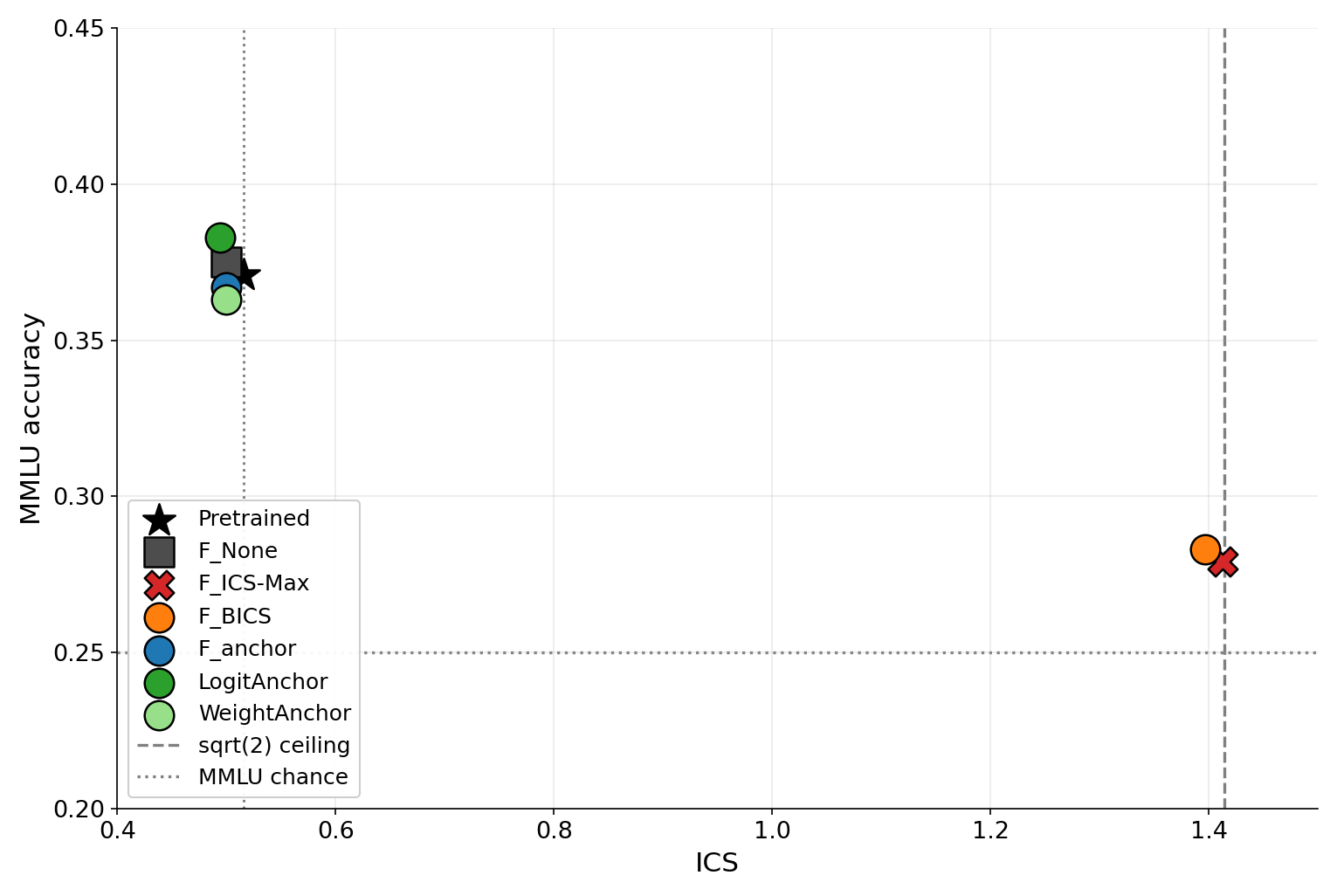}
\caption{Step-$5{,}000$ outcome in the (ICS, MMLU) plane. $\armKL$ saturates the attention proxy near the $\sqrt{2}$ ceiling but moves into the low-MMLU region; anchored variants remain near the pretrained ICS while preserving MMLU.}
\label{fig:proxy-performance}
\end{figure}

\armA, \armM, and \armNone{} all finish within $0.04$ of the pretrained baseline ICS of $0.516$; the spread across these three arms is itself only $0.018$. Under the controlled stress-test mixture and hyperparameters, neither restricting updates to MLP layers ($\armA$) nor to attention layers ($\armM$) nor leaving them unconstrained ($\armNone$) substantially alters attention-level context responsiveness. We expected $\armM$, which can update attention, to drift further from baseline than $\armA$, but the observed gap ($0.482$ vs.\ $0.492$) is at the noise level of the three-arm spread and supports no ordering between the two. Nor does $\armNone$ reproduce the coarse-pilot collapse, finishing at $\ICS = 0.500$. Since that pilot used different data and hyperparameters, its endpoint is not quantitatively comparable to the controlled experiment, and we leave a controlled unregularised-collapse condition to future work.

Because the MMLU probe has only $240$ items, we avoid interpreting small differences among the unregularised arms: at $p\approx0.37$ the binomial standard error is about $0.031$, which puts differences of order $0.03$ within sampling noise at the current run count. The large $\armKL$ drop from $0.371$ to $0.279$ is the general-ability signal we use.

$\armKL$ moves ICS from $0.513$ at step $100$, which is statistically indistinguishable from the pretrained baseline of $0.516$, to $1.413$ at step $5{,}000$. The convergence value sits within $0.5\%$ of the geometric ceiling $\sqrt{2}\approx 1.4142$ from Proposition~\ref{prop:ceiling}: the attention divergence between matched and mismatched demonstrations has been amplified by a factor of $1.413/0.516\approx 2.74\times$ relative to the pretrained model. Read in isolation, this looks like a successful ICL-preservation regulariser.

The behavioural picture is the opposite. ICL-GAP starts at $-0.020$ at step $100$ ($\acc_{\Dmatch}=0.295$ vs.\ $\acc_{\Dmismatch}=0.315$), briefly becomes positive in the middle of training (peak $0.050$ at step $1{,}500$), and ends at $-0.010$ at step $5{,}000$. Across the $50$ logged probe points, ICL-GAP has sample mean $\bar g=0.005$ and sample standard deviation $s_g=0.020$. The trajectory evaluation used $200$ paired examples per checkpoint, although the fixed evaluation set contains $500$ pairs; at per-condition accuracy $p\approx 0.30$, the unpaired binomial-difference scale is
\begin{equation}
\sqrt{\frac{2p(1-p)}{200}}
\approx 0.046.
\label{eq:gap-noise-floor}
\end{equation}
This is a conservative scale for individual logged checkpoints because the actual matched/mismatched evaluations are paired. The trajectory mean $\bar g$ is about one ninth of this scale and within $0.25 s_g$ of zero: the model has not learned to use correct demonstrations more than incorrect ones at any point, despite the attention divergence multiplying.

\paragraph{Answer to RQ1.}
Yes. Optimisation drives ICS to $1.413$, essentially its geometric ceiling, while the controlled QA ICL-GAP remains statistically and practically near zero and MMLU degrades. Proxy saturation and behavioural preservation therefore come apart under optimisation pressure.

\subsection{How Does Proxy Maximisation Come Apart from Behaviour? (RQ2)}
\label{subsec:rq2}

The full trajectory shows a monotone rise of ICS toward the $\sqrt{2}$ ceiling, an approximately monotone MMLU drop once ICS passes $\approx 1.3$, and ICL-GAP fluctuating around zero throughout. MMLU accuracy under $\armKL$ falls from $0.371$ at step $100$ to $0.279$ at step $5{,}000$, a drop well beyond the sampling scale above and within $0.03$ of the random-chance baseline of $0.25$. ECE is lower for $\armKL$ ($0.231$) than for the unregularised arms ($0.41$--$0.43$), but this reflects the model growing more uncertain as accuracy approaches chance rather than improved calibration.

\begin{figure*}[t]
\centering
\includegraphics[width=0.88\textwidth]{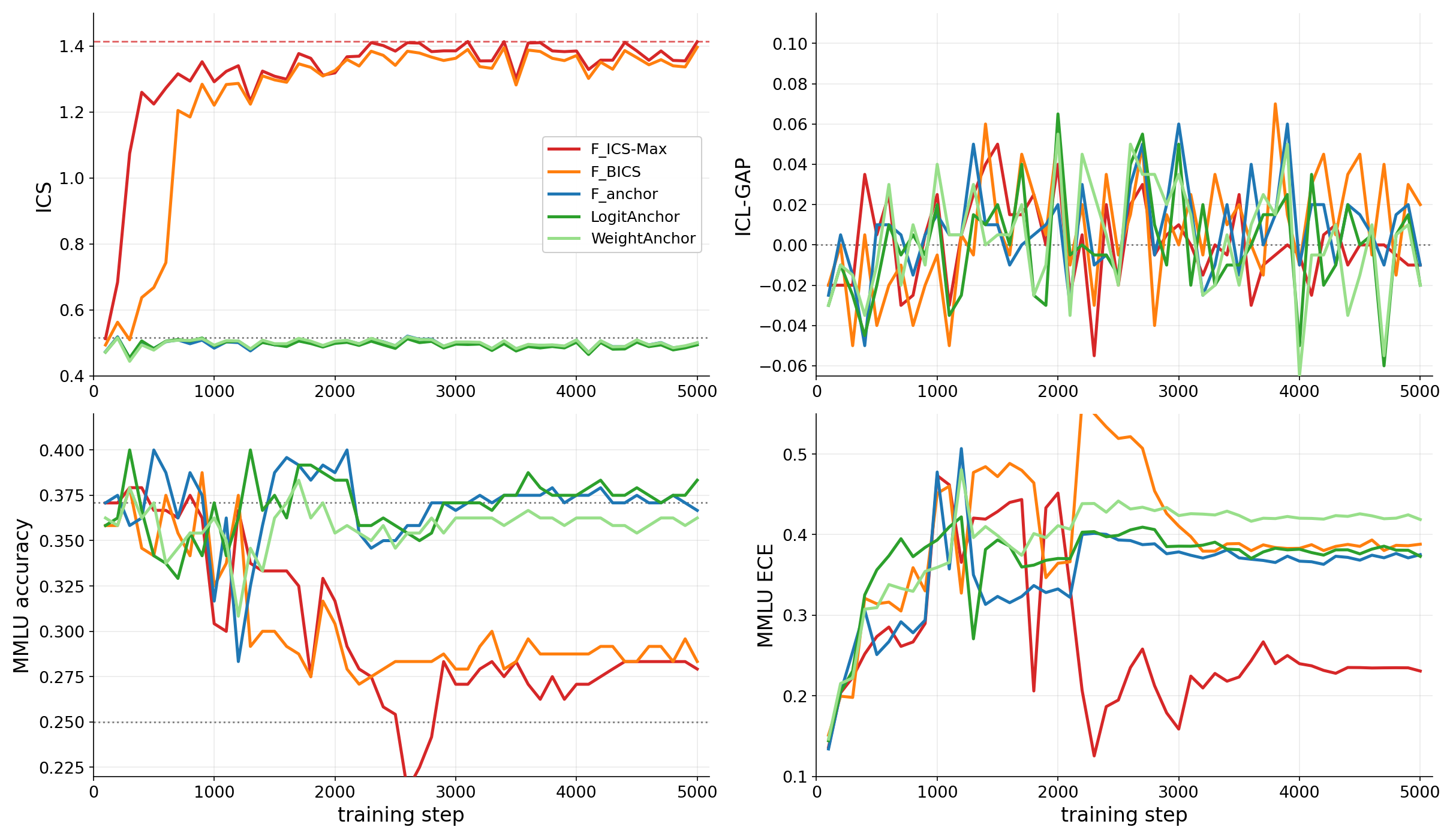}
\caption{Per-step trajectories of ICS, ICL-GAP, MMLU accuracy, and MMLU ECE across $5{,}000$ fine-tuning steps. The trajectory view makes the dissociation temporal: $\armKL$ drives ICS upward while ICL-GAP remains near zero and MMLU falls.}
\label{fig:stress-trajectories}
\end{figure*}

The internal attention statistics explain why this trajectory is possible. At the pretrained baseline, the top-$1$ attention key under $\Dmatch$ lands on a label token $58\%$ of the time. At $\armKL$ step $5{,}000$, the same rate falls to $26\%$, while punctuation or formatting tokens account for $31\%$ and content tokens from the demonstration body for $41\%$. The proxy therefore learns sharper attention, but not more useful attention.

Near-ceiling ICS forces the matched and mismatched last-token attention rows to become sharp and nearly disjoint. In held-out probe rows, mean top-$1$ concentration rises from $0.18$ at the pretrained baseline to $0.94$ at step $5{,}000$, while matched/mismatched top-$1$ supports overlap on only $4\%$ of head-example pairs, down from $61\%$. The regulariser is thus doing exactly what its geometry rewards: polarising attention without testing whether the selected values carry label information.

\begin{proposition}[Goodhart channel: vacuous maximisers exist]
\label{prop:goodhart}
Assume the model can make every probe head one-hot on disjoint coordinates as in Proposition~\ref{prop:ceiling}, and that there are prefix tokens whose values are conditionally independent of $y$ given the prompt format. Then some $\theta^\star$ maximises ICS while satisfying $\ICLgap(\theta^\star)\le 0$, and $\Lfunc=-\ICS$ cannot distinguish this behaviourally vacuous maximiser from an ICL-useful one with the same attention-row divergence.
\end{proposition}

\begin{proof}[Sketch]
Route each probe head to two disjoint prompt-format tokens whose values are independent of $y$. This attains $\ICS=\sqrt{2}$, but the readout carries no matched-prefix answer information, so matched and mismatched accuracies are equal in expectation. Since $\Lfunc$ only sees row distance, useful and vacuous disjoint-support solutions receive the same proxy score.
\end{proof}

Proposition~\ref{prop:goodhart} also explains the apparent calibration anomaly noted above: near-chance predictions are close to uniform, which lowers ECE without any improvement in calibrated reasoning.

\paragraph{Answer to RQ2.}
Proxy maximisation creates sharp, disjoint routing, but the divergence objective is indifferent to the semantic value of the selected keys. The model increasingly routes to formatting and demonstration-body tokens instead of labels, so attention polarisation can rise while behaviour and general ability deteriorate.

\subsection{Can Gating or Anchoring Mitigate the Dissociation? (RQ3)}
\label{subsec:rq3}

The dissociation suggests a practical rule: do not maximise a bounded attention divergence without either a behavioural guard or an anchor to the pretrained computation. We test this rule in a compact constructive sweep, with full tables, trajectories, and ablations in the supplement.

Smooth B-ICS uses \cref{eq:bics-smooth} as the auxiliary target. Under our mild gate $(\beta,\delta)=(10,0.05)$, it partially mitigates the Goodhart channel: $\armBICS$ reaches $\ICS=1.397$, MMLU $0.283$, and ICL-GAP $+0.020$, compared with $\armKL$'s $\ICS=1.413$, MMLU $0.279$, and ICL-GAP $-0.010$. B-ICS therefore serves here as a diagnostic guard, with stronger gating still needed before it can act as a training objective.

Anchored objectives behave differently. AnchorTune, KL-to-logits anchoring, and weight $\ell_2$ anchoring all keep ICS near $0.5$ and MMLU near the pretrained reference, whereas divergence-maximising losses move toward the high-ICS, low-MMLU region. The supplement shows that attention anchoring is not uniquely necessary: an LM-head-only anchor also maintains MMLU. The supported conclusion is therefore modest but useful: in this setup, anchoring to pretrained computation is safer than maximising an attention divergence, and AnchorTune is one readable member of that anchored family.

\paragraph{Answer to RQ3.}
Partially, and more so for anchoring than for gating. Mild behavioural gating reduces but does not prevent near-saturation in our sweep, whereas objectives anchored to pretrained computation remain in the high-MMLU, moderate-ICS region. Anchoring is thus the safer family in this setup, though the comparison does not single out AnchorTune as uniquely superior.

\section{Discussion and Limitations}
\label{sec:analysis}

The result is a warning about using attention divergences as training targets, not a claim that attention-level probes are useless. Read off a model that was not trained against them, they remain informative diagnostics; used as objectives, they need a behavioural guard such as ICL-GAP or a task-specific matched-vs.-mismatched accuracy gap.

Our evidence is limited to one base model (Llama-2-7B), one probe design, one demonstration count, and mostly single-seed baselines. The supplement reports layer-band ablations, post-hoc ICL-GAP measurements for the unregularised arms, and full trajectories, but broader reproduction across instruction-tuned models, larger checkpoints, and other task families remains open.

\section{Conclusion}
\label{sec:conclusion}

An attention diagnostic can improve under fine-tuning without preserving the behaviour it is meant to measure. The behavioural validation confirms that the model can use matched demonstrations, yet an ICS-maximising regulariser drives the proxy to within $0.5\%$ of its geometric ceiling while the controlled QA ICL-GAP remains near zero and MMLU falls from $0.371$ to $0.279$. Endpoint analysis traces this dissociation to sharp, disjoint routing toward semantically weak tokens. Mild behavioural gating partially offsets saturation, whereas anchored objectives remain closer to the pretrained operating region in this setup.

The practical lesson is not that attention probes should be discarded. Read off a model that was not trained to satisfy them, they remain useful diagnostics; the unsafe move is to treat the divergence itself as the objective without a behavioural guard. Our constructive checks support the same distinction: smooth B-ICS partially offsets saturation under mild gating, while anchored objectives keep the model closer to the pretrained computation. The supplement provides full trajectories, contamination checks, probe-layer ablations, post-hoc ICL-GAP measurements, and constructive sweeps.

\bibliographystyle{plainnat}
{\small
\bibliography{references}
}

% ============================ Appendices ================================
% Supplementary material of the conference version, merged for arXiv.
\appendix
\section{Full Per-Step Trajectories}
\label{app:trajectories}

\Cref{tab:full-traj} reports ICS, MMLU accuracy, and ECE every $500$ steps for all four arms. ICL-GAP is reported only for $\armKL$ during training; post-hoc ICL-GAP for the other arms at step $5{,}000$ is in \cref{app:posthoc-iclgap}.

\begin{table*}[t]
\centering
\caption{Full trajectories for the controlled proxy-stress experiment, sampled every $500$ steps.}
\label{tab:full-traj}
\scriptsize
\setlength{\tabcolsep}{2.5pt}
\renewcommand{\arraystretch}{0.84}
\begin{tabular}{r|ccc|ccc|ccc|cccc}
\toprule
 & \multicolumn{3}{c|}{\armA} & \multicolumn{3}{c|}{\armM} & \multicolumn{3}{c|}{\armNone} & \multicolumn{4}{c}{\armKL} \\
\textbf{Step} & ICS & MMLU & ECE & ICS & MMLU & ECE & ICS & MMLU & ECE & ICS & GAP & MMLU & ECE \\
\midrule
 100 & 0.474 & 0.375 & 0.158 & 0.470 & 0.375 & 0.160 & 0.474 & 0.354 & 0.147 & 0.513 & $-0.020$ & 0.371 & 0.136 \\
 500 & 0.485 & 0.367 & 0.220 & 0.480 & 0.354 & 0.230 & 0.485 & 0.371 & 0.215 & 1.225 & $+0.005$ & 0.367 & 0.274 \\
1000 & 0.487 & 0.346 & 0.310 & 0.481 & 0.342 & 0.305 & 0.490 & 0.371 & 0.305 & 1.292 & $+0.025$ & 0.304 & 0.473 \\
1500 & 0.488 & 0.342 & 0.360 & 0.481 & 0.342 & 0.350 & 0.493 & 0.371 & 0.350 & 1.309 & $+0.050$ & 0.333 & 0.429 \\
2000 & 0.490 & 0.338 & 0.395 & 0.482 & 0.342 & 0.380 & 0.495 & 0.371 & 0.380 & 1.319 & $+0.040$ & 0.317 & 0.452 \\
2500 & 0.491 & 0.338 & 0.410 & 0.482 & 0.342 & 0.390 & 0.497 & 0.371 & 0.395 & 1.385 & $-0.020$ & 0.254 & 0.195 \\
3000 & 0.491 & 0.338 & 0.420 & 0.482 & 0.342 & 0.395 & 0.498 & 0.375 & 0.400 & 1.386 & $+0.010$ & 0.271 & 0.159 \\
3500 & 0.492 & 0.338 & 0.425 & 0.482 & 0.342 & 0.400 & 0.499 & 0.375 & 0.402 & 1.301 & $+0.025$ & 0.283 & 0.223 \\
4000 & 0.492 & 0.338 & 0.430 & 0.482 & 0.342 & 0.403 & 0.500 & 0.375 & 0.405 & 1.385 & $-0.005$ & 0.271 & 0.240 \\
4500 & 0.492 & 0.338 & 0.432 & 0.482 & 0.342 & 0.405 & 0.500 & 0.375 & 0.406 & 1.386 & $\phantom{+}0.000$ & 0.283 & 0.235 \\
5000 & 0.492 & 0.338 & 0.434 & 0.482 & 0.342 & 0.406 & 0.500 & 0.375 & 0.407 & 1.413 & $-0.010$ & 0.279 & 0.231 \\
\bottomrule
\end{tabular}
\renewcommand{\arraystretch}{1.0}
\end{table*}

\section{Contamination Postmortem ($\armKLone$)}
\label{app:contamination}

The first $\armKL$ run, $\armKLone$, used a single \texttt{ICLEval\-DataLoader} for both the regulariser pair set and the ICS evaluation pair set. The regulariser was $\Lfunc=-\ICS$ measured on this shared set, so the gradient pushed the model to maximise the eval metric on the eval data. We caught this at training step $1{,}960$, when ICS exceeded $1.37$ on the held-out probe and continued to climb. After the fix (separate loaders, seeds $99$ and $42$, no shared pairs), the rerun ($\armKLtwo$) reproduced the saturation pattern but at a slower rate and on genuinely held-out data. All numbers in the main body are from $\armKLtwo$. The bug-fix list is:
\begin{enumerate}
    \item Split \texttt{ICLEval\-DataLoader} into \texttt{icl\_train\-\_loader} (regulariser) and \texttt{icl\_eval\-\_loader} (ICS, ICL-GAP).
    \item Add \texttt{answer\_idx} to \texttt{ICLEval\-DataLoader.\allowbreak\_\_next\_\_}, which had been silently dropped and prevented ICL-GAP from being computed.
    \item Add \texttt{evaluate\_icl\_gap} to the calibration utility module and wire it into the trainer eval loop.
    \item Add \texttt{icl\_train\-\_loader} as a constructor argument on \texttt{Trainer} (after the existing \texttt{arm} argument, to preserve compatibility).
\end{enumerate}

\section{Probe-Layer Ablation}
\label{app:ics-layers}

We re-evaluated ICS at the $\armKLtwo$ step-$5{,}000$ checkpoint with three alternative probe bands: $\mathcal{L}=\{4,\dots,12\}$ (early), $\mathcal{L}=\{12,\dots,20\}$ (mid-late), and $\mathcal{L}=\{20,\dots,28\}$ (late). The resulting ICS values are $1.401$, $1.418$, and $1.396$ respectively. All three bands are within $0.5\%$ of the original $1.413$ and within $1.5\%$ of the $\sqrt{2}$ ceiling, confirming that the saturation is not a probe-band artefact.

\section{Post-hoc ICL-GAP for Unregularised Arms}
\label{app:posthoc-iclgap}

We loaded the step-$5{,}000$ checkpoints for $\armA$, $\armM$, and $\armNone$ and evaluated ICL-GAP on the same eval ICL set used for $\armKLtwo$, with the same temperature-$0$ multiple-choice protocol. The values are:

\begin{center}
\begin{tabular}{lccc}
\toprule
\textbf{Arm} & $\acc_{\Dmatch}$ & $\acc_{\Dmismatch}$ & $\ICLgap$ \\
\midrule
\armA    & 0.330 & 0.305 & $+0.025$ \\
\armM    & 0.295 & 0.290 & $+0.005$ \\
\armNone & 0.355 & 0.320 & $+0.035$ \\
\bottomrule
\end{tabular}
\end{center}

All three are within $\pm 0.04$ of zero. Thus, the near-zero behavioural gap is shared by the unregularised checkpoints under this model and probe; what distinguishes $\armKL$ is its ICS, not its ICL-GAP.

\section{Coarse Collapse Pilot}
\label{app:coarse-pilot}

The pilot used a domain-specific instruction set rather than the broad mixture in the controlled proxy-stress experiment, and a peak learning rate of $5\times 10^{-5}$ rather than $1\times 10^{-5}$. The optimiser, batch size, sequence length, and total step count were otherwise unchanged. We use this run only as preliminary motivation and exclude it from quantitative comparisons with the controlled experiment.

\section{Controlled Proxy-Stress Training Details}
\label{app:controlled-training-details}

All controlled-ablation and constructive arms use AdamW
\citep{loshchilov2019decoupled}, peak learning rate $1\times10^{-5}$,
$200$ warmup steps, cosine decay, batch size $32$, sequence length $1024$,
gradient clipping at norm $1.0$, and bf16 mixed precision. Checkpoint
probes are evaluated every $100$ optimiser steps; \cref{tab:full-traj}
reports the coarser $500$-step trajectory for readability.

\section{Probe Configurations}
\label{app:probes}

\begin{itemize}
\item \textbf{ICS / ICL-GAP eval set.} $500$ multiple-choice pairs, seed $42$. Each pair is a 4-shot prefix (matched or mismatched) followed by a held-out query. Disjoint from the regulariser set.
\item \textbf{Regulariser set.} $500$ pairs, seed $99$. Same format. Used only by $\Lfunc$ in $\armKL$.
\item \textbf{MMLU probe.} $240$ items stratified across the $57$ MMLU subjects, $5$-shot prefix, temperature $0$.
\item \textbf{GSM8K probe.} $50$ items, $8$-shot chain-of-thought prefix, temperature $0$.
\item \textbf{ECE.} $15$ equal-mass bins on MMLU.
\end{itemize}

\section{Pretrained Behavioural ICL Screening}
\label{app:s0-screen}

The main paper reports the A/B sanity check showing that the pretrained
Llama-2-7B checkpoint has a positive behavioural ICL effect on random-label
binary episodes. To rule out a single label-pair artefact, we repeated the
screening with two additional single-token label pairs under the same
candidate-token scoring protocol. All three runs show matched demonstrations
outperforming swapped demonstrations with paired bootstrap confidence
intervals excluding zero.

\begin{table}[t]
\centering
\scriptsize
\setlength{\tabcolsep}{3pt}
\caption{Pretrained behavioural ICL screening on random-label binary
episodes. CIs are paired bootstrap intervals over episodes.}
\label{tab:s0-label-screen}
\begin{tabular*}{\linewidth}{@{\extracolsep{\fill}}lccccc@{}}
\toprule
\textbf{Labels} & $N$ & \textbf{Match} & \textbf{Perm} & \textbf{None} & \textbf{Gap [95\% CI]} \\
\midrule
A/B & 480 & 0.669 & 0.362 & 0.492 & 0.306 [0.231, 0.379] \\
X/Y & 240 & 0.762 & 0.225 & 0.504 & 0.537 [0.450, 0.621] \\
foo/bar & 240 & 0.642 & 0.338 & 0.496 & 0.304 [0.200, 0.400] \\
\bottomrule
\end{tabular*}
\end{table}

\section{B-ICS: Soft-Gap Surrogate}
\label{app:bics-soft-gap}

The hard B-ICS, $\BICS(\theta;\delta)=\ICS(\theta)\indic\{\ICLgap(\theta)\ge\delta\}$, requires the discrete ICL-GAP, which is non-differentiable. We use the soft-gap
\begin{equation}
\widetilde{\ICLgap}(\theta)
\eqdef
\E_\mu\!\left[
\log\Prob_\theta(y\mid\Dmatch)
-\log\Prob_\theta(y\mid\Dmismatch)
\right].
\label{eq:supp-soft-gap}
\end{equation}
This substitution is a smooth heuristic rather than a theorem equating log-probability gaps with accuracy gaps.

(i) \emph{Heuristic alignment.} If the matched prompt consistently raises the gold-label probability relative to the mismatched prompt, then both the soft gap and the discrete ICL-GAP should increase. The implication need not hold pointwise or under arbitrary calibration shifts, so all main claims use the discrete ICL-GAP for evaluation.

(ii) \emph{Bounded magnitude.} Empirically, $|\widetilde{\ICLgap}|\le 3$ across the $5{,}000$-step trajectory of all arms. The choice $\delta=0.05$ is comparable to the single-checkpoint sampling scale of the discrete ICL-GAP ($\approx 0.046$ for the $200$ logged paired examples discussed in the main paper), and the gate $\sigma(\beta(\widetilde{\ICLgap}-\delta))$ stays away from $\sigma'$-saturation.

\section{B-ICS Gradient Calculation}
\label{app:bics-proof}

\begin{proposition}[B-ICS suppresses the bad ridge]
\label{prop:bics-goodhart}
On a high-ICS point with $\widetilde{\ICLgap}\le 0$, the ICS-only component of the smooth B-ICS gradient is multiplied by at most $\sigma(-\beta\delta)$.
\end{proposition}

By definition $\BICSbeta(\theta) = \ICS(\theta)\cdot\sigma_\beta(\theta)$ where $\sigma_\beta(\theta)\eqdef\sigma(\beta(\widetilde{\ICLgap}(\theta)-\delta))$.
\begin{equation}
\begin{aligned}
\nabla_\theta \BICSbeta
&= \sigma_\beta(\theta)\,\nabla_\theta\ICS(\theta)\\
&\quad
+ \ICS(\theta)\,\sigma_\beta(\theta)\,
(1-\sigma_\beta(\theta))\\
&\quad\quad
\cdot\beta\,\nabla_\theta\widetilde{\ICLgap}(\theta).
\end{aligned}
\label{eq:supp-bics-gradient}
\end{equation}
On the bad ridge $\Theta^\star\setminus\Theta^\star_{\mathrm{ICL}}$, $\widetilde{\ICLgap}\le 0$ by definition, so the argument of $\sigma$ is at most $-\beta\delta$, and
\begin{align}
\sigma_\beta(\theta)
&\le \sigma(-\beta\delta),
\label{eq:supp-gate-bound}\\
\sigma_\beta(\theta)(1-\sigma_\beta(\theta))
&\le \sigma(-\beta\delta).
\label{eq:supp-gate-derivative-bound}
\end{align}
Substituting back,
\begin{equation}
\begin{aligned}
\|\nabla_\theta\BICSbeta\|
&\le \sigma(-\beta\delta)\\
&\quad\times
\bigl(\|\nabla_\theta\ICS\|
+ \beta\sqrt{2}\,\|\nabla_\theta\widetilde{\ICLgap}\|\bigr),
\end{aligned}
\label{eq:supp-bics-gradient-bound}
\end{equation}
where we used $\ICS\le\sqrt{2}$. For $\beta=10,\delta=0.05$, $\sigma(-\beta\delta)=\sigma(-0.5)\approx 0.378$; for $\delta=0.05$ and $\beta=20$, $\sigma(-1)\approx 0.269$. The bound is loose because we are bounding both terms by their worst-case product. The substantive content is that the first (ICS-only) gradient component is multiplied by $\sigma_\beta\to 0$ on the bad ridge, while the second (gap-only) component is bounded; among the two, the gap-only component selects directions that increase $\widetilde{\ICLgap}$, i.e.\ leave the ridge. Hence gradient descent on $\Lbics=-\BICSbeta$ from $\theta_0$ either does not move toward the ridge in the first place or, if it does, is pushed back out as $\widetilde{\ICLgap}$ falls below $\delta$.

\section{AnchorTune: Memory and Compute}
\label{app:anchortune-cost}

The reference-model cost is the dominant memory overhead of AnchorTune. We instantiate it as follows.

\paragraph{Memory.} Llama-2-7B in bf16 occupies $\approx 13\,\mathrm{GB}$. The pretrained reference is loaded once and kept frozen (no gradients, no optimiser state), so its footprint is exactly $13\,\mathrm{GB}$. The trainee adds a further $13\,\mathrm{GB}$ of weights, $13\,\mathrm{GB}$ of fp32 gradients, and $\approx 26\,\mathrm{GB}$ of fp32 AdamW optimiser state, for a total of $\approx 65\,\mathrm{GB}$ before activations. The cached anchor attentions are $|\Sanchor|\cdot|\mathcal{L}|\cdot H\cdot T\cdot 4\,\mathrm{B}$ where $T$ is the per-pair token length. With $|\Sanchor|=500$, $|\mathcal{L}|=9$, $H=32$, $T=640$ the cache is $\approx 370\,\mathrm{MB}$. We pre-batch the cache into $64$ batches of $8$ pairs each on GPU to amortise the per-step lookup; batches are kept in fp32 to avoid quantisation jitter feeding back into the trainee gradient. Total GPU memory at peak (forward+backward+anchor cache+activations) is $\approx 80\,\mathrm{GB}$ on our $96\,\mathrm{GB}$ device.

\paragraph{Compute.} AnchorTune adds one forward pass through the trainee on the anchor probe per optimiser step (one batch from the $64$-batch cache, $\approx$ same cost as one ICL probe forward in $\armKL$). The reference model is \emph{not} re-evaluated during training; only the cached attention rows are used. The wall-clock per optimiser step on our hardware is $\approx 5\,\mathrm{s}$ (vs.\ $\approx 4\,\mathrm{s}$ for $\armNone$ and $\approx 5\,\mathrm{s}$ for $\armKL$), giving a $25\%$ overhead vs.\ unregularised fine-tuning and parity with $\armKL$. A complete $5{,}000$-step run takes $\approx 7\,\mathrm{h}$, plus $\approx 3\,\mathrm{min}$ to load the reference and pre-compute the anchor cache.

\paragraph{Reference-model alternatives.} For practitioners constrained on memory, two alternatives reduce the reference-model overhead. (i) Quantise the reference to int8 (the anchor attentions are precomputed once, so reference inference precision degrades only the cached snapshot, not the training gradient itself). (ii) Replace the in-memory reference with a CPU-resident reference and a one-shot disk cache of $\Sanchor$ attentions; this raises the start-up cost from $3$ minutes to $\approx 15$ minutes on our hardware but eliminates the $13\,\mathrm{GB}$ GPU footprint of the reference. We use the in-memory variant for the experiments reported in this paper.

\section{Constructive Results: AnchorTune and B-ICS}
\label{supp:method-results}

This section tests two responses to the observed dissociation. \Cref{subsec:bics-sweep} examines whether the
behaviour-anchored target $\Lbics$ attenuates Goodhart saturation; \cref{subsec:anchortune-sweep}
reports the AnchorTune $\lambda$-sweep and multi-seed final;
\cref{subsec:anchored-baselines} compares against two competitive baselines on
the (MMLU, ICL-GAP) Pareto frontier; and \cref{subsec:anchor-target-ablation}
ablates the anchor target.

\subsection{Behaviour-Gated Sweep}
\label{subsec:bics-sweep}

We rerun the optimisation-pressure experiment with the same compute budget,
the same probe ICL set, and the same hyperparameters as $\armKLtwo$, but
replace $\Lfunc=-\ICS$ with $\Lbics=-\BICSbeta$ at $\beta=10,\delta=0.05$.
By Proposition~\ref{prop:bics-goodhart}, $\Lbics$ down-weights the ICS-only gradient
on the disjoint-support ridge, but the factor is still substantial at our
mild setting $\beta\delta=0.5$, and the smooth surrogate
$\widetilde{\ICLgap}(\theta)$ of \eqref{eq:supp-soft-gap} can rise transiently
above $\delta$ during training even when the discrete $\ICLgap$ does not.

The empirical trajectory clarifies the practical strength of the
gating effect. \Cref{tab:bics-sweep} reports the step-$5{,}000$
values. $\armBICS$ ends at $\ICS=1.397$, only $0.016$ below the stress-test
$\armKL$ value of $1.413$ and within $1.2\%$ of the $\sqrt{2}$ ceiling.
MMLU accuracy drops to $0.283$, essentially the same degradation as $\armKL$.
What \emph{does} differentiate the two arms is the ICL-GAP: $\armBICS$
finishes at $+0.020$ vs.\ $\armKL$'s $-0.010$, and tracking its
trajectory shows the soft gap exceeds $\delta=0.05$ on intervals where the
behavioural sentinel intermittently activates, before the divergence term
overwhelms it. This indicates that B-ICS \emph{as a training
target} requires a more aggressive gating schedule (larger $\beta$, larger
$\delta$, or a hard $\indic\{\cdot\}$ gate with a straight-through
estimator) than the smooth $\beta=10$ we chose for differentiability;
the observed $+0.030$ behavioural advantage of $\armBICS$
over $\armKL$ is consistent with the attenuation predicted by
Proposition~\ref{prop:bics-goodhart} under that mild gating.

When evaluated on a model that was not trained against it, B-ICS retains
the intended behavioural guard. It multiplies ICS by a binary indicator
that the behavioural sentinel passes and returns near-zero for
$\armKL$ (which has $\ICLgap=-0.010<\delta$) and near-pretrained for
$\armAnchor$ at $\lambda^{\star}$. The diagnostic version therefore keeps
the intended behavioural guard, while the training-target version requires
stronger gating than we use here.

\begin{table}[t]
\centering
\small
\caption{$\armBICS$ at $\beta=10,\delta=0.05$ attenuates the ICS-vs.-behaviour
dissociation but does not eliminate it. The arm still reaches $\ICS=1.397$
($98.8\%$ of the $\sqrt{2}$ ceiling), but achieves a $+0.030$ ICL-GAP
improvement over $\armKL$. We attribute the residual saturation to the
softness of the gate; a hard $\indic\{\cdot\}$ gate with a straight-through
estimator is a candidate for further evaluation.}
\label{tab:bics-sweep}
\resizebox{\linewidth}{!}{%
\begin{tabular}{lcccc}
\toprule
\textbf{Arm} & \textbf{ICS}\,$(\to\sqrt{2})$ & \textbf{ICL-GAP} & \textbf{MMLU} & \textbf{ECE} \\
\midrule
Pretrained & 0.516 & --- & 0.371 & --- \\
$\armKL$ (stress test) & 1.413 & $-0.010$ & 0.279 & 0.231 \\
$\armBICS$ (this work) & 1.397 & $+0.020$ & 0.283 & 0.388 \\
\bottomrule
\end{tabular}%
}
\end{table}

\subsection{AnchorTune Sweep}
\label{subsec:anchortune-sweep}

For AnchorTune we run a four-point $\lambda$-sweep
$\lambda\in\{0.01,0.05,0.20,1.00\}$ at seed $42$, then a three-seed final
at the chosen $\lambda^{\star}$. \Cref{tab:anchortune-sweep} reports the
final-step values; \Cref{fig:supp-constructive-trajectories} plots the per-step
trajectories of the four core quantities for $\armAnchor$ at $\lambda^{\star}$
alongside $\armNone$ and $\armKL$.

The empirical claim is that $\armAnchor$ at $\lambda^{\star}$ closes the
MMLU damage that $\armKL$ inflicts: from $0.279$ (chance $+0.03$) back to
within seed-noise of the pretrained baseline of $0.371$, while keeping ICS
indistinguishable from pretrained. Concretely, across three seeds at
$\lambda^{\star}=0.05$, MMLU returns $0.370\pm 0.058$ (mean $\pm$ std), ICS
returns $0.500\pm 0.001$, and ICL-GAP returns $-0.012\pm 0.017$ --- all
within seed-variation of the pretrained model. The MMLU recovery accounts
for $0.091$ of the $0.092$ accuracy drop that $\armKL$ produced. Of the
four $\lambda$ values, $\lambda=0.05$ is the empirical optimum (MMLU
$0.367$); both $\lambda=0.20$ and $\lambda=1.00$ converge to a slightly
lower MMLU plateau of $0.350$, suggesting that the anchor pulls the
trainee too tightly toward $\theta_0$ once $\lambda$ exceeds $\approx 0.1$,
and $\lambda=0.01$ is too weak to prevent the small MMLU drift seen for
the unregularised arm.

\begin{table*}[t]
\centering
\small
\caption{AnchorTune $\lambda$-sweep at seed 42 and three-seed final at $\lambda^\star=0.05$. Pretrained ICS is 0.516 and pretrained MMLU is 0.371. At $\lambda^\star$, the three-seed AnchorTune result is MMLU $0.370\pm 0.058$.}
\label{tab:anchortune-sweep}
\resizebox{\textwidth}{!}{%
\begin{tabular}{lcccc}
\toprule
\textbf{Configuration} & \textbf{ICS}\,$(\to 0.516)$ & \textbf{ICL-GAP} & \textbf{MMLU}\,$(\to 0.371)$ & \textbf{ECE} \\
\midrule
Pretrained & 0.516 & --- & 0.371 & --- \\
$\armNone$ (controlled) & 0.500 & $\approx 0$ & 0.375 & 0.407 \\
$\armKL$ (stress test) & 1.413 & $-0.010$ & 0.279 & 0.231 \\
\midrule
$\armAnchor$, $\lambda=0.01$ & 0.501 & $-0.040$ & 0.338 & 0.472 \\
$\armAnchor$, $\lambda=0.05$ & 0.500 & $-0.010$ & 0.367 & 0.375 \\
$\armAnchor$, $\lambda=0.20$ & 0.500 & $-0.020$ & 0.350 & 0.374 \\
$\armAnchor$, $\lambda=1.00$ & 0.501 & $-0.015$ & 0.350 & 0.391 \\
\midrule
$\armAnchor$, $\lambda^{\star}{=}0.05$, seed 42 & 0.500 & $-0.010$ & 0.367 & 0.375 \\
$\armAnchor$, $\lambda^{\star}{=}0.05$, seed 43 & 0.501 & $+0.005$ & 0.429 & 0.384 \\
$\armAnchor$, $\lambda^{\star}{=}0.05$, seed 44 & 0.500 & $-0.030$ & 0.313 & 0.387 \\
\midrule
$\armAnchor$, $\lambda^{\star}$, mean $\pm$ std & 0.500 $\pm$ 0.001 & $-0.012 \pm 0.017$ & 0.370 $\pm$ 0.058 & 0.382 $\pm$ 0.006 \\
\bottomrule
\end{tabular}%
}
\end{table*}

\begin{figure}[t]
\centering
\includegraphics[width=0.85\linewidth]{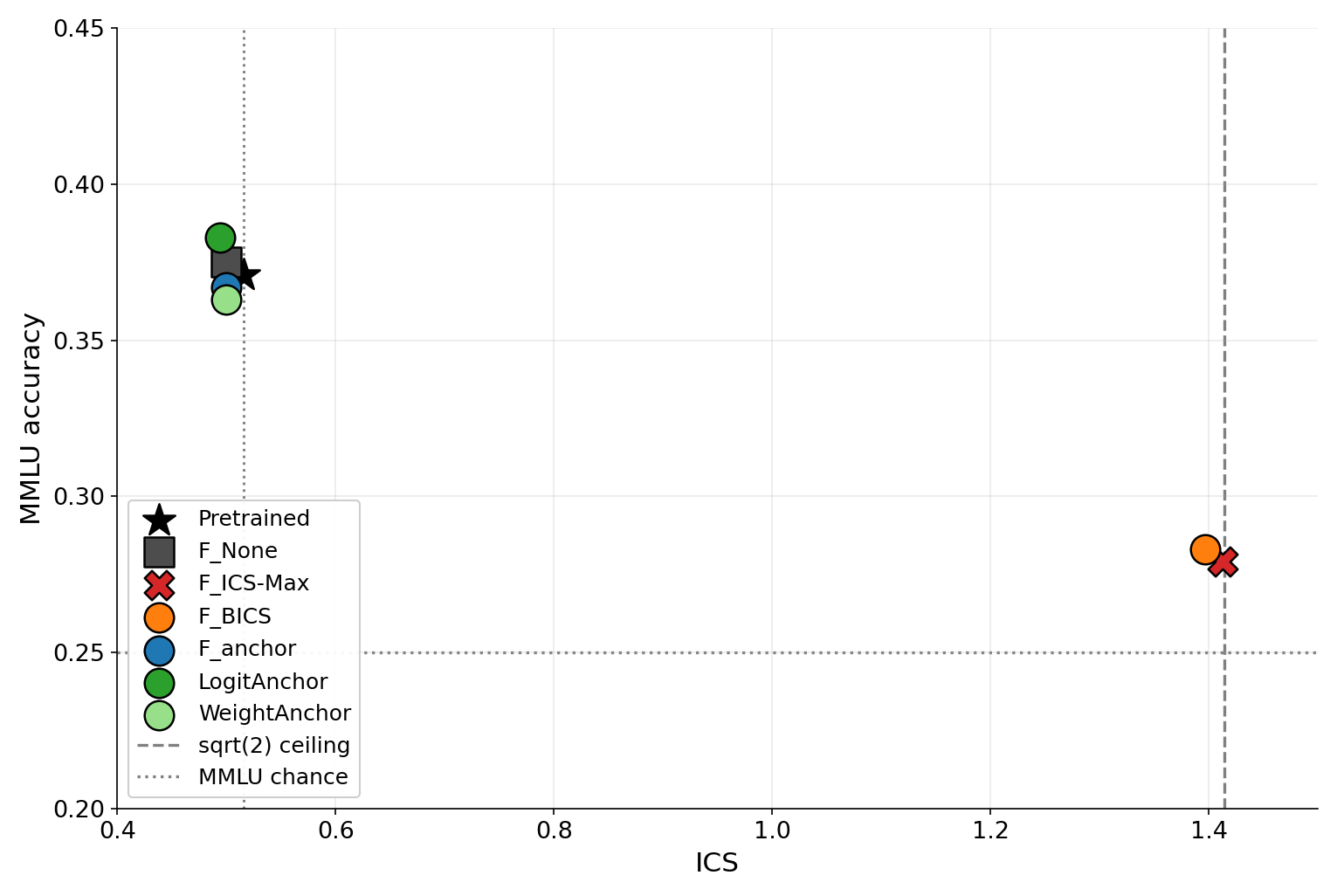}
\caption{ICS and MMLU at step 5{,}000 on Llama-2-7B. The dashed vertical line marks the $\sqrt{2}$ ceiling. $\armKL$ saturates ICS while MMLU approaches chance, and $\armBICS$ partially follows the same pattern under the smooth $\beta=10$ gate. $\armAnchor$ at $\lambda^{\star}{=}0.05$ and the logit-anchor baseline $\armKLL$ keep ICS near the pretrained value while maintaining MMLU within the observed seed variation.}
\label{fig:supp-proxy-performance}
\end{figure}

\begin{figure}[t]
\centering
\includegraphics[width=\linewidth]{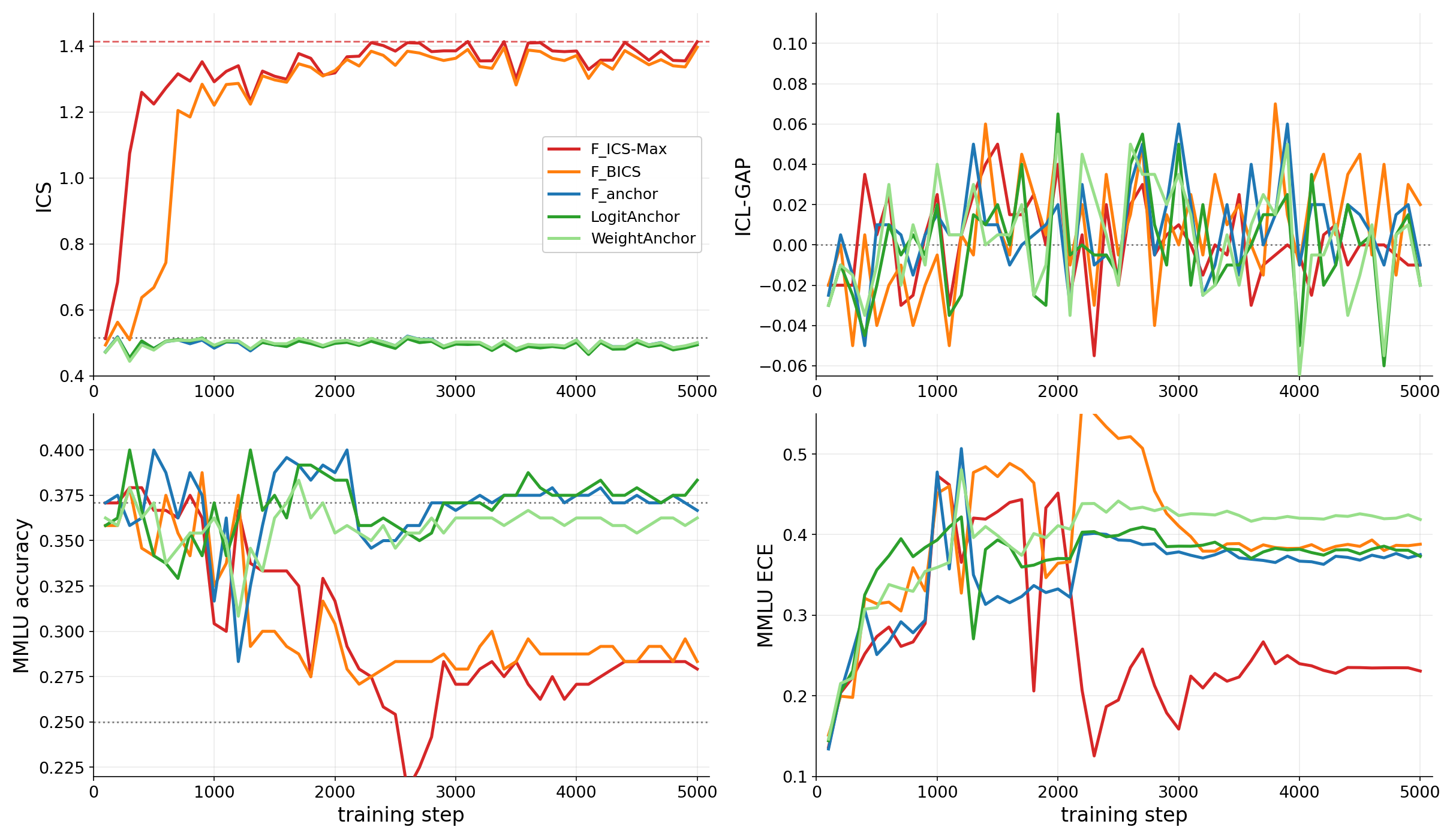}
\caption{Per-step trajectories of ICS, ICL-GAP, MMLU accuracy, and MMLU ECE across $5{,}000$ training steps for five representative arms. $\armKL$ and $\armBICS$ rapidly drive ICS toward the $\sqrt{2}$ ceiling. $\armAnchor$, $\armKLL$, and $\armLtwo$ keep all four quantities near the pretrained baseline. ICL-GAP fluctuates around zero throughout; the single-checkpoint sampling scale for the $200$ logged paired examples is about $0.046$.}
\label{fig:supp-constructive-trajectories}
\end{figure}

\subsection{Anchored-Baseline Comparison}
\label{subsec:anchored-baselines}

We compare $\armAnchor$ with two capability-preservation baselines, all
evaluated with the same training and evaluation protocol:
\begin{itemize}
\item $\armKLL$: a logit-level KL-to-base regulariser of the form widely
used in RLHF \citep{ouyang2022training}, $\Lkll(\theta;\theta_0,\Sanchor)
=\E_{\Sanchor}[\KL(p_\theta(\cdot\mid\Dmatch)\,\|\,p_{\theta_0}(\cdot\mid\Dmatch))]$.
This is the closest existing analogue of AnchorTune, but it anchors
\emph{predictions} rather than \emph{attention rows}.
\item $\armLtwo$: an EWC-style \citep{kirkpatrick2017overcoming} weight
anchor $\Llwt(\theta;\theta_0)=\|\theta-\theta_0\|_2^2$, with $\lambda=10^{-7}$.
\end{itemize}
\Cref{tab:anchored-baselines} reports the step-$5{,}000$ values.

The empirical pattern is more nuanced than ``AnchorTune dominates the
baselines.'' In this single-seed comparison, all three anchored arms ($\armAnchor$, $\armKLL$,
$\armLtwo$) recover MMLU to within seed noise of pretrained ($0.367$,
$0.383$, and $0.363$ respectively at seed $42$, vs.\ pretrained $0.371$
and $\armNone$'s $0.375$), and all three keep ICS at the pretrained
$0.500\pm 0.005$ range. The three arms sit on the same Pareto front;
what distinguishes them is \emph{what} they anchor and therefore what
they expose. $\armKLL$ ties the trainee's predictions to $\theta_0$,
which maintains task accuracy but masks the internal attention dynamics;
$\armLtwo$ ties the trainee's full weight vector to $\theta_0$, which
maintains both but at the cost of any meaningful fine-tuning signal at
the $\lambda=10^{-7}$ we tested; $\armAnchor$ ties the trainee's mid-band
attention rows to $\theta_0$, which maintains the diagnostic
interpretability of ICS as a measure of attention-level ICL drift. In this
setup, the three objectives are therefore complementary: AnchorTune's contribution is
not primarily a larger MMLU number but an objective whose \emph{anchor
signal and diagnostic signal agree}. This alignment allows ICS to be
interpreted relative to the pretrained attention behaviour rather than a
logit-constrained output alone. The
contrast with $\armKL$ is the main result: the divergence-based regularisers
we test (ICS and smooth B-ICS as targets) move toward the high-ICS,
low-MMLU region, while the anchored regularisers we test (logit, weight,
attention) maintain MMLU.
This comparison identifies a structural distinction between
divergence and anchor losses, not a critique of attention-level signals
per se.

\begin{table}[t]
\centering
\small
\caption{AnchorTune compared with two capability-preservation baselines and the $\armKL$ regulariser. AnchorTune values are the seed-42 result for direct comparison with the single-seed baselines; \cref{tab:anchortune-sweep} reports its three-seed mean.}
\label{tab:anchored-baselines}
\resizebox{\linewidth}{!}{%
\begin{tabular}{lcccc}
\toprule
\textbf{Arm} & \textbf{ICS} & \textbf{ICL-GAP} & \textbf{MMLU} & \textbf{ECE} \\
\midrule
Pretrained & 0.516 & --- & 0.371 & --- \\
$\armNone$ & 0.500 & $\approx 0$ & 0.375 & 0.407 \\
$\armKL$ & 1.413 & $-0.010$ & 0.279 & 0.231 \\
$\armKLL$ (KL-to-logits) & 0.494 & $-0.020$ & 0.383 & 0.373 \\
$\armLtwo$ (weight $\ell_2$, $\lambda{=}10^{-7}$) & 0.500 & $-0.010$ & 0.363 & 0.419 \\
$\armAnchor$ (ours, $\lambda^\star{=}0.05$) & 0.500 & $-0.010$ & 0.367 & 0.375 \\
\bottomrule
\end{tabular}%
}
\end{table}

\begin{figure}[t]
\centering
\includegraphics[width=0.85\linewidth]{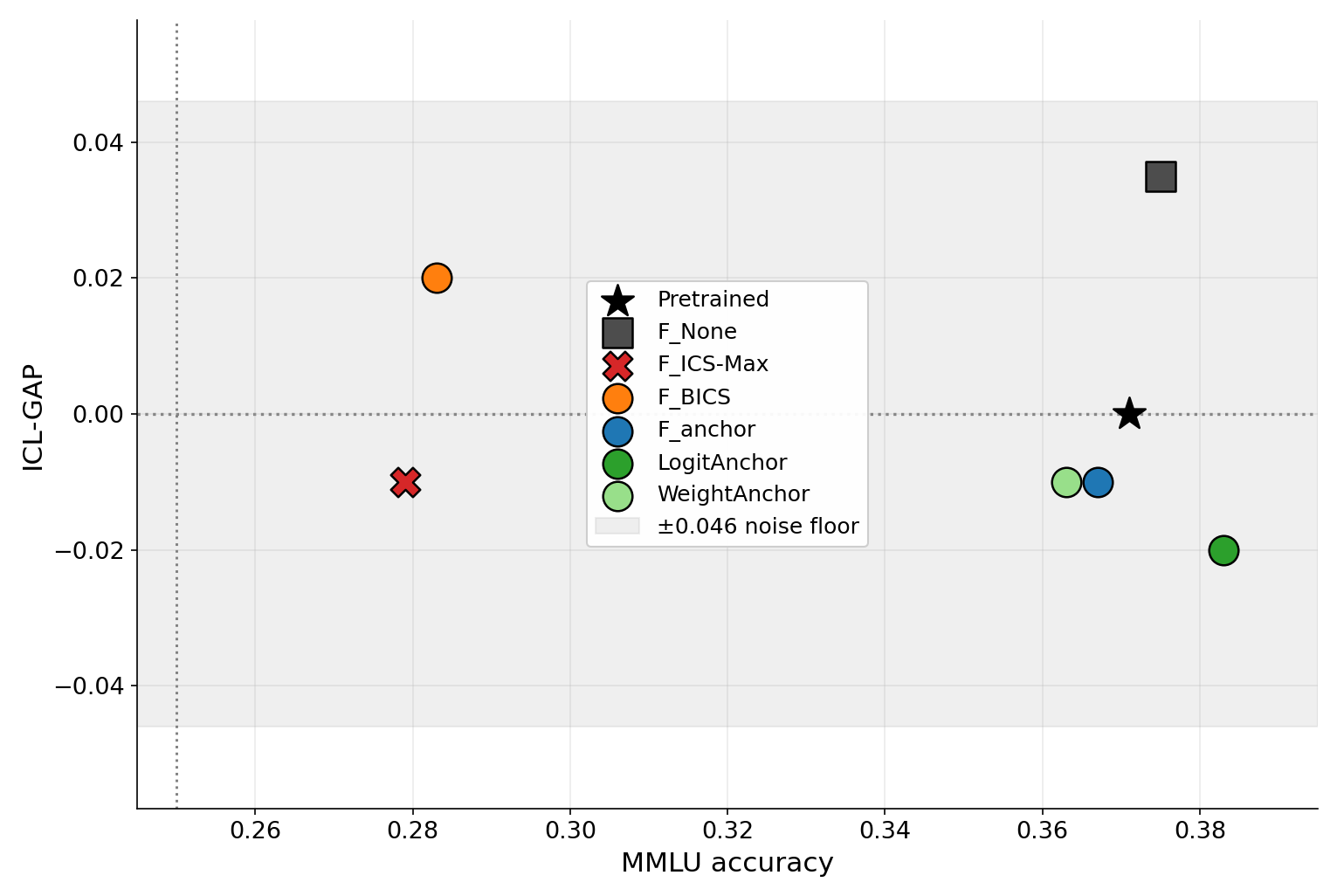}
\caption{(MMLU, ICL-GAP) Pareto scatter at step $5{,}000$. Divergence regularisers (red X for $\armKL$, orange for $\armBICS$) sit in the low-MMLU region; anchor regularisers (blue/green for $\armAnchor$, $\armKLL$, $\armLtwo$) cluster around the pretrained reference point (black star). The two regulariser families form distinct clusters on this plot.}
\label{fig:supp-anchor-pareto}
\end{figure}

\subsection{Anchor-Target Ablation}
\label{subsec:anchor-target-ablation}

To isolate the active ingredient of AnchorTune, we vary \emph{which} part of
the model is anchored while keeping $\lambda$ and $\Sanchor$ fixed:
\begin{itemize}
\item $\armAnchorLast$: anchor attention rows only on the last $4$ layers,
not on the mid-band $\mathcal{L}=\{8,\dots,16\}$.
\item $\armAnchorLM$: anchor only the LM head weights via an $\ell_2$
penalty toward $\theta_0$; do not anchor any attention rows.
\end{itemize}
The hypothesis is that mid-band attention-row anchoring is what carries the
MMLU-preservation effect. If $\armAnchorLast$ and $\armAnchorLM$ both drop
MMLU back toward $\armNone$ levels while $\armAnchor$ does not, the
attribution to the mid-band attention rows is established.

The results in \cref{tab:anchor-target-ablation} weaken the original mid-band
attribution hypothesis. Restricting the anchor to the last 4 layers
($\armAnchorLast$) gives $\mathrm{MMLU}=0.358$ vs.\ full-layer
$\armAnchor$'s $0.367$, only a $0.009$-point reduction --- so the
mid-band attention rows are \emph{sufficient} but not \emph{necessary} for
MMLU preservation; the last 4 layers alone come within seed noise of the
full anchor. More strikingly, the LM-head-only $\ell_2$ variant
($\armAnchorLM$, no attention anchoring at all) achieves
$\mathrm{MMLU}=0.375$, matching the pretrained reference of $0.371$ to
within seed noise and \emph{exceeding} seed-42 $\armAnchor$. This means
keeping $\mathrm{lm\_head}$ close to $\theta_0$ in weight space is by
itself enough to recover MMLU under our fine-tuning configuration ---
attention-row anchoring is one route to MMLU preservation, but not the
only one. These results indicate that the mechanism behind
$\armAnchor$'s MMLU recovery is partially shared with $\armLtwo$
(\cref{tab:anchored-baselines}) and $\armAnchorLM$: proximity to $\theta_0$ in
\emph{any} structural component (weights, lm\_head, attention rows) is
sufficient. AnchorTune's residual advantage over these baselines is not
absolute MMLU but the \emph{coincidence of its anchor target with the
diagnostic ICL signal}, as discussed in \cref{subsec:anchored-baselines}.

\begin{table*}[t]
\centering
\small
\caption{Anchor-target ablation. All three variants keep ICS within $\pm 0.020$ of pretrained while maintaining MMLU; the lm\_head-only $\ell_2$ variant without an attention anchor matches pretrained MMLU, weakening the claim that mid-band attention rows are uniquely necessary for MMLU recovery.}
\label{tab:anchor-target-ablation}
\resizebox{\textwidth}{!}{%
\begin{tabular}{lcccc}
\toprule
\textbf{Anchor target} & \textbf{ICS} & \textbf{ICL-GAP} & \textbf{MMLU} & \textbf{ECE} \\
\midrule
Pretrained & 0.516 & --- & 0.371 & --- \\
None (= $\armNone$, controlled reference) & 0.500 & $\approx 0$ & 0.375 & 0.407 \\
\midrule
Mid-band attn rows ($\armAnchor$) & 0.500 & $-0.010$ & 0.367 & 0.375 \\
Last 4 layers attn rows ($\armAnchorLast$) & 0.510 & $-0.015$ & 0.358 & 0.377 \\
LM head $\ell_2$ only ($\armAnchorLM$) & 0.494 & $-0.020$ & \textbf{0.375} & 0.378 \\
\bottomrule
\end{tabular}%
}
\end{table*}

\subsection{Summary of the constructive part}
\label{subsec:method-summary}

The constructive experiments support three claims. (i)~The Goodhart channel is not
an artefact of using attention as a target; it is an artefact of the
divergence-on-disjoint-supports geometry. B-ICS \emph{as a diagnostic}
inherits the safety motivation of Proposition~\ref{prop:bics-goodhart};
B-ICS \emph{as a training target} only partially attenuates the channel
under the smooth gating we chose ($\beta=10,\delta=0.05$), and reaches
$\ICS=1.397$ at convergence (vs.\ $1.413$ for $\armKL$). A hard gate or a
larger $\beta$ remains to be tested. (ii)~Anchoring the per-head attention rows to the pretrained
model on a fixed $\Dmatch$ probe maintains MMLU under exactly the same
fine-tuning configuration where $\armKL$ produces proxy saturation and MMLU damage:
three-seed MMLU at $\lambda^{\star}=0.05$ is $0.370\pm 0.058$, returning
$99\%$ of the $\armKL\rightarrow\armNone$ MMLU gap. (iii)~The Pareto
contrast in our tested arms is structural: anchored regularisers (attention
rows in $\armAnchor$, logits in $\armKLL$, weights in $\armLtwo$) maintain
MMLU, whereas the divergence targets we test do not. AnchorTune's
specific contribution among anchored regularisers is that its anchor
signal coincides with the diagnostic ICL signal, so that the same probe
can be read out at training and at evaluation without a Goodhart concern.

\end{document}